\def\year{2021}\relax
\documentclass[letterpaper]{article} 
\usepackage{aaai21} 
\usepackage{times} 
\usepackage{helvet} 
\usepackage{courier} 
\usepackage[hyphens]{url} 
\usepackage{graphicx} 
\usepackage{graphicx} 
\usepackage{natbib} 
\usepackage{caption} 
\title{Fast Decomposition of Temporal Logic Specifications for Heterogeneous Teams\thanks{\small{DISTRIBUTION STATEMENT A. Approved for public release. Distribution is unlimited.
This material is based upon work supported by the Under Secretary of Defense for Research and Engineering under Air Force Contract No. FA8702-15-D-0001. Any opinions, findings, conclusions or recommendations expressed in this material are those of the author(s) and do not necessarily reflect the views of the Under Secretary of Defense for Research and Engineering.}}}
\author{
Kevin Leahy,\textsuperscript{\rm 1} 
Austin Jones,\textsuperscript{\rm 1}
Cristian Ioan Vasile,\textsuperscript{\rm 2}\\
}
\affiliations{
\textsuperscript{\rm 1}MIT Lincoln Laboratory, Lexington, MA\\
\textsuperscript{\rm 2}Lehigh University, Bethlehem, PA
}

\usepackage{multicol}
\usepackage[bookmarks=true]{hyperref}
\usepackage{comment}
\usepackage{float}
\usepackage{pdflscape}

\usepackage{graphicx} 
\usepackage{subcaption}
\usepackage{amsmath} 
\usepackage{amssymb}  
\usepackage[ruled,vlined,linesnumbered]{algorithm2e}
\usepackage{mathtools}
\usepackage{amsthm}
\usepackage{tikz}
\graphicspath{{Images/}}

\theoremstyle{remark}
\newtheorem{dfn}{Definition}
\newtheorem{prop}{Proposition}
\newtheorem{example}{Example}
\newtheorem{problem}{Problem}
\newtheorem{remark}{Remark}

\newcommand \until[1]{\mathcal{U}_{[#1)}}
\newcommand \even[1]{\diamondsuit_{[#1)}}
\newcommand \alw[1]{\square_{[#1)}}

\usetikzlibrary{arrows.meta}
\tikzset{%
  >={Latex[width=2mm,length=2mm]},
            base/.style = {rectangle, rounded corners, draw=black,
                           minimum width=4cm, minimum height=1cm,
                           text centered, font=\sffamily},
  activityStarts/.style = {base, fill=blue!30},
       startstop/.style = {base, fill=red!30},
    activityRuns/.style = {base, fill=green!30},
         process/.style = {base, minimum width=2.5cm, fill=orange!15,
                           font=\ttfamily},
}

\begin{document}

\maketitle

\begin{abstract}
In this work, we focus on decomposing large multi-agent
path planning problems with global temporal logic goals (common to all agents) into
smaller sub-problems that can be solved and executed
independently.
Crucially, the sub-problems' solutions must jointly satisfy the common global mission specification.
The agents' missions are given as Capability Temporal Logic (CaTL) formulas,
a fragment of signal temporal logic, that can express properties over tasks involving multiple agent capabilities (sensors, e.g., camera, IR, and effectors, e.g., wheeled, flying, manipulators) under strict timing constraints.
The approach we take is to decompose both the temporal logic specification and the team of agents.
We jointly reason about the assignment of agents to subteams and the decomposition of formulas using a satisfiability modulo theories (SMT) approach.
The output of the SMT is then distributed to subteams and
leads to a significant speed up in planning time.
We include computational results to evaluate the efficiency of our solution, as well as the trade-offs introduced by the conservative nature of the SMT encoding.

\end{abstract}

\section{Introduction}
\label{sec:intro}

The problem of planning for a large team of heterogeneous agents from a high-level specification remains difficult. An attractive approach for ameliorating the difficulty of such problems is to decompose the problem into sub-problems that can be solved and executed in parallel. This approach carries trade-offs. Specifically, there is generally a large up-front computational cost to determine a feasible decomposition of the task and team, in exchange for faster planning and execution. In this work, we introduce a system for quickly decomposing temporal logic specifications to allow large teams of heterogeneous agents to plan in near real-time.

Recently there has been strong interest in planning for teams of agents from temporal logic specifications.
Works in this area include coordination of very large groups of homogeneous agents as a swarm~\cite{haghighi2016,chen2018} or using sampling~\cite{kantaros2020}. The current work is designed to enable planning for very large groups of heterogeneous agents. Other works that consider heterogeneous agents feature agent-specific specifications~\cite{GuoTaskMotion2017,GuoMultiAgent2015}, whereas we are concerned strictly with global specifications.
Some related works consider teams of heterogeneous agents whose missions feature time-abstract semantics~\cite{sahin2019,schillinger2018}. 
Our goal is to design a system that can work with teams of heterogeneous agents with concrete timing requirements.

This work is most directly related to~\citet{Jones2019ScRATCHS}. That work introduced Capability Temporal Logic (CaTL), a fragment of signal temporal logic. CaTL is designed for tasking large teams of heterogeneous agents, each with varying capability to service requests. A centralized mixed integer linear program (MILP) is used to generate a plan for the entire team simultaneously from a given CaTL specification. Here, we seek to improve the computational time of the solution in~\citet{Jones2019ScRATCHS} by decomposing the specification and team of agents into sub-specifications and subteams. Such a decomposition generates several smaller MILPs that can be solved in parallel, rather than one large MILP. The decomposition also helps with decentralized execution of the specification, potentially reducing communication and computational burden during mission execution.

Another closely related work is~\citet{chen2012}. In that approach, each agent had an associated set of capabilities that it could service. The mission specification was projected onto a language for each agent that captured the requests it could service. The product of the language of the entire set of agents was then checked for trace-closedness, to determine if the mission could be decomposed. In our work, we are interested in large teams, so a product language of projections is not computationally tractable. Rather than find the complete set of solutions, we focus on finding a feasible agent-to-task pairing quickly. Likewise,~\citet{schillinger2018decomposition} is focused on decomposition of time-abstract specifications using product a automaton for each agent and designing a compact team automaton.
Such an approach would not be tractable with the concrete-time specifications we consider in this work.
Other work that is concerned with decomposition of temporal logic specifications includes~\citet{banks2020}. That work uses an cross-entropy optimization approach for task allocation for a team of homogeneous agents. We are concerned in this work with task allocation for a heterogeneous team.

In this work, we define conditions for parallelizing a CaTL specification. We also provide encodings for solving them as a satisfiability modulo theories (SMT) problem~\cite{barrett2018}. The specification and agent assignment are encoded in an SMT problem, and a set of sub-specifications and subteams is returned. Each sub-specification/subteam pair can then be solved as a MILP using the methods presented in~\citet{Jones2019ScRATCHS}.

\section{Models and Specifications}
\label{sec:def}

In this section, we introduce the models for the environment and agents, and the specification language, Capability Temporal Logic (CaTL)~\cite{Jones2019ScRATCHS}, for describing behaviors of these systems.

\paragraph{Environment}
We consider a team of agents operating in an environment consisting of a finite set of discrete locations (states) $Q$ and weighted edges $\mathcal{E}$ between states, where the weights represent positive integer travel times.
The states are labeled with atomic propositions from a set $AP$.
We denote the labeling function by $L: Q \to 2^{AP}$.

\paragraph{Agents}
Each agent has a set of capabilities it can execute.
We denote the finite set of all agent capabilities by $Cap$, and the set of all agents by $J$.

\begin{dfn}
An \emph{Agent} $j\in J$ is given by a tuple $A_j = (q_{0,j},Cap_j)$, where $q_{0,j} \in Q$ is the initial state of the agent and $Cap_j \subseteq Cap$ is its set of capabilities.
\end{dfn}
\begin{dfn}
The  motion of agent $A_j$ in the environment induces a \emph{trajectory}, denoted $s_j:\mathbb{Z}_{\geq 0} \to Q \cup \mathcal{E}$, such that $s_j(0) = q_{0,j}$ and $s_j(t)$ returns the state or edge occupied by agent $A_j$ at time $t \in \mathbb{Z}_{\geq 1}$.
\end{dfn}

We denote the number of agents with capability $c \in Cap$ at state $q \in Q$ and time $t \in \mathbb{Z}_{\geq 0}$ by $n_{q, c}(t) = |\{j\in J \mid q=s_j(t), c\in Cap_j\}|$.
A partition $\{J_\ell\}_{\ell\in \mathcal{J}}$ of the agent set $J$ is called a \emph{team partition}.

\begin{dfn}
The \emph{synchronous trajectory} $s_{J'}$ obtained from a set of agent trajectories $\{s_j\}_{j \in J'}$ with $J' \subseteq J$ is given by
$s_{J'}  = \bigcup_{j \in J'} s_{j}$.
\end{dfn}


\paragraph{Capability Temporal Logic}

The team of agents is tasked with a high-level specification given as a CaTL formula.
Here, we define the syntax and semantics of CaTL. 
The atomic unit of a CaTL formula is a \emph{task}.


\begin{dfn}
A \emph{task} is a tuple $T = (d,\pi, cp)$,
where $d \in \mathbb{Z}$ is a duration of time,
$\pi\in AP$ is an atomic proposition,
$cp: Cap \to \mathbb{Z}_{\geq 1} \cup \{-\infty\}$ is a counting map specifying how many agents with each capability should be in each region labeled $\pi$.
A capability $c$ that is not required to perform task $T$ is defined by $cp(c)=-\infty$.
We abuse notation, and denote the set of required capabilities for a task $T$ by $cp_T \neq \emptyset$.
\end{dfn}


CaTL is a fragment of STL~\cite{maler2004}, where the core units are tasks rather than arbitrary predicates.


\begin{dfn}
The \emph{syntax} of CaTL~\cite{Jones2019ScRATCHS} is
\begin{equation*}
    \phi ::=  T \mid \phi_1 \land \phi_2 \mid \phi_1 \lor \phi_2 \mid \phi_1 \until{a,b} \phi_2 \mid \even{a,b} \phi \mid \alw{a,b} \phi
\end{equation*}
where $\phi$ is a CaTL formula, $T$ is a task, $\land$ and $\lor$ are the Boolean conjunction and disjunction operators, $\until{a,b}$, $\even{a,b}$, and $\alw{a,b}$ are the time-bounded until, eventually, and always operators, respectively.
\end{dfn}


\begin{dfn}
The \emph{qualitative semantics} of CaTL are defined over synchronous trajectories $s_{J}$. At time $t$,
\begin{multline*}
    (s_{J},t) \models T \Leftrightarrow  \forall \tau \in [t,t+d), \forall q \in L^{-1}(\pi), \forall c \in Cap\\
    n_{q, c}(\tau) \geq cp(c),
\end{multline*}
while the remaining semantics are defined as for STL~\cite{maler2004}.
A team trajectory satisfies a CaTL formula $\phi$, denoted $s_{J} \models \phi$, if $(s_{J},0) \models \phi$.
\end{dfn}

\begin{dfn}[Quantitative Semantics (Availability Robustness)]
\label{def:quantSem}
The  availability robustness of a task is computed as
\begin{equation}
\label{eq:rho_rules}
\rho_{a}(s_{J},t,T) = \underset{c\in Cap}{\min}\,  \underset{t' \in [t,t+d)}{\min}  \underset{q \in L^{-1}(\pi)}{\min}  n_{q,c}(t') - cp(c)
\end{equation}
while for the other operators it is computed recursively as for STL~\cite{maler2004}.
\end{dfn}



\section{Problem Statement}

Denote by $\mathrm{Synth}(J', \phi)$ a method that returns trajectories $s_j$ for all $j \in J' \subseteq J$ such that the team trajectory $s_{J'}$ maximizes the availability robustness $\rho_a(s_{J'}, 0, \phi)$ with respect to $\phi$.
We are now ready to formally state the decomposition problem.

\begin{problem}
\label{prob:decomposition}
Given a set of agents $\{A_j\}_{j \in J}$ and a CaTL formula $\phi$, find a team partition $R$, and a set of formulas $\{\phi_r\}_{r \in R}$, if the synthesis problem is feasible. Formally,
\begin{equation}
\begin{array}{lcl}
s^*_J \models \phi & \Rightarrow &
\left( s_{J_r} \models \phi_r, \ \forall r \in R \right) \land \left(  s_J \models \phi \right)
\end{array},
\end{equation}
where $s^*_J = \bigcup_{j\in J} s^*_j$,
$\{s^*_j\}_{j\in J} = \mathrm{Synth}(J, \phi)$,
$s_J = \bigcup_{r\in R} s_{J_r}$,
and $\{s_j\}_{j \in J_r} = \mathrm{Synth}(J_r, \phi_r)$ for all $r \in R$.
\end{problem}




To solve Problem~\ref{prob:decomposition}, we use syntax trees corresponding to CaTL formulas (Sec.~\ref{sec:ast}). We propose sufficient conditions to decompose formulas (Sec.~\ref{sec:sufCond}), and encode these conditions as constraints in an SMT problem (Sec.~\ref{sec:sat}).
The SMT solution is an assignment of agents to tasks. This assignment is used to decompose the specification into a set of sub-specifications, each with a corresponding subteam.
The resulting decomposed specifications and teams can each be handled concurrently using $\mathrm{Synth}$.
There are many choices for the synthesis method $\mathrm{Synth}$~\cite{kress2018,belta2017}.
We employ an MILP approach similar to~\cite{Jones2019ScRATCHS}.
The decomposition process is outlined in Algorithm~\ref{alg:overview}.

\begin{algorithm}
\KwIn{Agents $\lbrace A_j\rbrace_{j\in J}$, CaTL formula $\phi$}
\KwOut{Team partition $\{J_r\}_{r \in R}$, Formulas $\{\phi_r\}_{r \in R}$}
\DontPrintSemicolon
$SMT\gets \mathrm{EncodeSMT}(\lbrace A_j\rbrace_{j\in J},\phi)$\;
Assignment $\alpha \gets \mathrm{Solve}(SMT)$\;
$\{\phi_r\}_{r\in R},\:\{J_r\}_{r \in R}\gets \mathrm{DecomposeTree}(\phi,\alpha)$\;
\Return{$\{J_r\}_{r \in R}$, $\{\phi_r\}_{r \in R}$}
\caption{Solution overview.}
\label{alg:overview}
\end{algorithm}

\section{Syntax Trees and Agent Assignments}
\label{sec:ast}

We use syntax trees to reason about CaTL specifications, their decomposition, and the assignment of agents to tasks.

\begin{dfn}
A \emph{CaTL syntax tree} a CaTL formula $\phi$ is a tuple
$\mathcal{T}_{\phi} = (V, v_0, Par)$,
where $V = \{\wedge, \vee, \until{a,b}, \even{a,b}, \alw{a,b}, T=(d, \pi, cp)\}$ is the set of nodes associated the operators and tasks of $\phi$,
$v_0$ is a root node,
$Par: V \to (V \setminus \{v_0\}) \cup \{{\bowtie}\}$ is the bijective parent map that defines the tree structure, and $Par(v_0) = {\bowtie}$ denotes that the root has no parent.
\end{dfn}

Let $Ch(v)$ denote the set of all children of a node $v \in V$,
$\phi(v)$ the subformula of $\phi$ associated with node $v$,
and $\Lambda$ the set of leaf nodes, i.e., nodes without children $Ch(v) = \emptyset$.
The leaf nodes correspond to the tasks of CaTL formulas.




\begin{example}

The syntax tree for the formula 
\begin{equation}
\label{examSpec}
\psi =   (T_1 \lor T_2) \land ( (T_3)\, \until{a,b} (T_4 \land T_5)))
\end{equation}
is given in Fig.~\ref{absSyntTree}

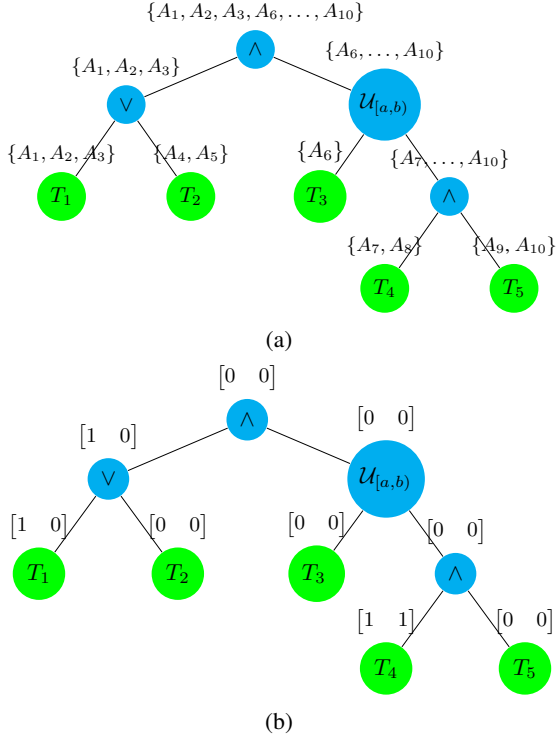
\begin{figure}
\begin{subfigure}[b]{0.9\columnwidth}
\centering 
\resizebox{\linewidth}{!}{
\begin{tikzpicture}[scale=.6,level 1/.style = {sibling distance = 20em,
    draw, align=center},
    level 2/.style={sibling distance=10em, 
    draw, align=center, level distance = 2.5cm},
     level 3/.style = {sibling distance = 10em,  
    draw, align=center, level distance = 2.5cm},
    level 4/.style={sibling distance=10em,
    draw, align=center}]
    \node [shape=circle, fill=cyan,label={\small$\{A_1,A_2,A_3,A_6,\ldots, A_{10}\}$}]{$\wedge$}
        child { node [shape=circle, fill=cyan,label={\small$\{A_1,A_2,A_3\}$}] {$\vee$} 
            child { node [shape=circle, fill=green,label={\small$\{A_1,A_2,A_3\}$}] {$T_1$} }
            child { node [shape=circle, fill=green,label={\small$\{A_4,A_5\}$}] {$T_2$} } }
        child { node [shape=circle, fill=cyan,label={\small$\{A_6,\ldots,A_{10}\}$}] {$\until{a,b}$}
            child { node [shape=circle, fill=green,label={$\small\{A_6\}$}] {$T_3$ } } 
            child { node [shape=circle, fill=cyan,label={\small$\{A_7,\ldots,A_{10}\}$}]  {$\wedge$} 
                child { node [shape=circle, fill=green,label={\small$\{A_7,A_8\}$}] {$T_4$}}
                child { node [shape=circle,fill=green,label={\small$\{A_9$, $A_{10}\}$}] {$T_5$}}}};
              
\end{tikzpicture}}
\caption{\label{absSyntTree}}
\end{subfigure}
\begin{subfigure}[b]{0.9\columnwidth}
\centering 
\resizebox{\linewidth}{!}{
\begin{tikzpicture}[scale=.6,level 1/.style = {sibling distance = 20em, 
    draw, align=center},
    level 2/.style={sibling distance=10em, 
    draw, align=center, level distance = 2.4cm},
     level 3/.style = {sibling distance = 10em,  
    draw, align=center, level distance = 2.4cm},
    level 4/.style={sibling distance=10em,
    draw, align=center, level distance = 2.4cm}]
    \node [shape=circle, fill=cyan,label={\small$\begin{bmatrix} 0 & 0\end{bmatrix}$}]{$\wedge$}
        child { node [shape=circle, fill=cyan,label={\small$\begin{bmatrix} 1 & 0\end{bmatrix}$}] {$\vee$} 
            child { node [shape=circle, fill=green,label={\small$\begin{bmatrix} 1 & 0\end{bmatrix}$}] {$T_1$} }
            child { node [shape=circle, fill=green,label={\small$\begin{bmatrix} 0 & 0\end{bmatrix}$}] {$T_2$} } }
        child { node [shape=circle, fill=cyan,label={\small$\begin{bmatrix} 0 & 0\end{bmatrix}$}] {$\until{a,b}$}
            child { node [shape=circle, fill=green,label={\small$\begin{bmatrix} 0 & 0\end{bmatrix}$}] {$T_3$ } } 
            child { node [shape=circle, fill=cyan,label={\small$\begin{bmatrix} 0 & 0\end{bmatrix}$}]  {$\wedge$} 
                child { node [shape=circle, fill=green,label={\small$\begin{bmatrix} 1 & 1\end{bmatrix}$}] {$T_4$}}
                child { node [shape=circle,fill=green,label={\small$\begin{bmatrix} 0 & 0\end{bmatrix}$}] {$T_5$}}}};
              
\end{tikzpicture}}
\caption{\label{exCap} }
\end{subfigure}
\caption{Abstract syntax tree for~\eqref{examSpec} with assignment~(\ref{absSyntTree}) and capability excess~(\ref{exCap})}
\end{figure}
\end{example}

\begin{dfn}
An \emph{assignment} of agents $\{A_j\}_{j \in J}$ in a CaTL syntax tree $\mathcal{T}_\phi$ is a mapping $\alpha: V \to 2^J$.
\end{dfn}

An assignment $\alpha$ keeps track of agents assigned to a tasks and subformulas.
Assignments are consistent with the formula structure, such that agents assigned to an intermediate node $v$ must be assigned to some child node of $v$.
Thus, $\alpha$ is completely determined by the assignment over the leaves $\Lambda$
\begin{equation}
    \alpha(v) = \bigcup_{v' \in Ch(v)} \alpha(v')\;. \\
\end{equation}

We further define the notion of \emph{capability excess} to aid in evaluating the assignment of agents to tasks.
\begin{dfn}
 The \emph{capability excess} of an assignment $\alpha$ to a node $v \in V$ is defined recursively as
\begin{equation}
\label{eq:ceRec}
ce(\alpha,v) = \begin{cases}
[na_{c} - cp(c)]_{c \in Cap} & v = T\\
ce(\alpha,\beta(Ch(v))) & v = \lor\\
[ \underset{v' \in Ch(v)}\min ce(\alpha,v')_c]_{c \in Cap} & \text{otherwise}
\end{cases}
\end{equation}
where
$na_{c_i} = |\{j\in \alpha(v) \mid c \in Cap_j\}|$ is the number of agents with capability $c_i$ assigned to node $v$, and
$\beta:2^V \to V$ is a deterministic selection map between the children of disjunction nodes.
\end{dfn}
In Sec.~\ref{sec:sat}, a specific instantiation of $\beta$ is denoted as $Choose$.



\begin{example}
Consider tasks involving two capabilities $Cap = \{c_1, c_2\}$.
The capabilities assigned to each agent, and the ones required by each task are listed in Fig.~\ref{expTab}.
An assignment of agents to tasks and the resulting capability excess are illustrated in Fig.~\ref{absSyntTree} and~\ref{exCap}, respectively.
\end{example}


\begin{figure}
  \centering
  \resizebox{0.95\columnwidth}{!}{
    \begin{tabular}{|c|c|c|}
    \hline
    Agent & $c_1$ & $c_2$ \\ \hline
        $A_1$ & 1 & 1  \\ \hline
        $A_2$ & 1 & 0 \\ \hline
        $A_3$ & 1 & 0\\ \hline
        $A_4$ & 1 & 1\\ \hline
        $A_5$ & 1 & 1\\ \hline
    \end{tabular}
    \begin{tabular}{|c|c|c|}
    \hline
    Agent & $c_1$ & $c_2$ \\ \hline
        $A_6$ & 0 & 1 \\ \hline
         $A_7$ & 1 & 1 \\ \hline
        $A_8$  & 1 & 1\\ \hline
        $A_9$ & 1 & 0 \\ \hline
         $A_{10}$ & 0 & 1 \\ \hline         
    \end{tabular}
    \quad
    \begin{tabular}{|c|c|c|}
    \hline
    Task & $cp(c_1)$ & $cp(c_2)$ \\ \hline
        $T_1$ & 2  & $-\infty$   \\  \hline
        $T_2$ & 2 & 2  \\ \hline
        $T_3$ & $-\infty$ & 1 \\ \hline
        $T_4$ & 1 & 1 \\ \hline
        $T_5$ &  1 &  1 \\ \hline
    \end{tabular}}
    \caption{\label{expTab} Numbers and requirements of capabilities for~\eqref{examSpec}.}
    \end{figure}

\begin{dfn}\label{def:elg}
 An assignment $\alpha$ is \emph{eligible} for the CaTL syntax tree $\mathcal{T}_\phi$, denoted $\alpha \models_e \mathcal{T}_{\phi}$, if 
 \begin{equation}
 \underset{c \in Cap}{\min} ce(\alpha,v_0)_c \geq 0
 \end{equation}
\end{dfn}

\begin{prop}
For a given team trajectory  $s_J$, let $\alpha[s_J]$ denote the induced assignment such that $A_j \in \alpha(\lambda)$ if $A_j$ participates in task $\lambda \in \Lambda$.  
It holds that
\begin{equation}
s_J \models \phi \Rightarrow \alpha[s_J] \models_e \mathcal{T}_\phi
\end{equation}
That is, eligibility is a necessary condition for satisfiability.
\end{prop}
\begin{proof}
The robustness $\rho_a$ would reach its theoretical maximum if 
all agents were able to simultaneously service all tasks.
Denote the team trajectory resulting from these conditions as $s_{0,J}$.
We now prove by induction that $ce(\alpha[s_{0,J}], v) \geq \rho_a(s_{0,J},t,\phi(v))$.

\textbf{Base Case.} As $\rho_a(s_{0,J},t,T)$ is defined as a minimum over capabilities and over regions, it would reach its maximum if all of the agents with required capabilities $cp_T$ were equally divided among the required regions 
at the appropriate time.
This maximum value is given by $ce(\alpha[s_{0,J}],T)$, i.e., $ce(\alpha[s_{0,J}],T) \geq  \rho_a(s_{0,J},t,T), \forall t$.  

\textbf{Recursion.} For $\wedge$ and $\vee$, we apply the same recursive relationships to $ce$ as we do to $\rho_a$.
Thus, if $ce(\alpha[s_{0,J}],v')$, $v' \in Ch(v)$, are upper bounds,
then $ce(\alpha[s_{0,J}], v) \geq \rho_a(s_{0,J},t, \phi(v)=\bigotimes_{v'\in Ch(v)} \phi(v')), \forall t , \bigotimes \in \{\land,\lor\}$.

For the temporal operators, since $s_{0,J}$ considers the case when agents are not motion or time-constrained, we can ignore the maximization and minimization with respect to temporal arguments in the recursive semantics.
This yields the form of the recursive relations in~\eqref{eq:ceRec}.
Therefore, if $ce(\alpha[s_{0,J}],v')$, $v' \in Ch(v)$, are upper bounds, then $ce(\alpha[s_{0,J}], v) > \rho_a(s_{0,J},t,\phi(v))$ $\forall t$,
where $\phi(v)$  is $\phi_1 \until{a,b} \phi_2$, $\even{a,b}\phi_1$, or $\alw{a,b} \phi_1$.

Thus, $ce(\alpha[s_{0,J}],v_0) \geq \rho_a(s_{0,J},0,\phi(v_0))$, $\forall \phi \in CaTL$.
Further, $s_J\models\phi \Rightarrow \rho_a(s_{J},0,\phi) \geq 0 \Rightarrow ce(\alpha[s_{0,J}],\phi)\geq 0$, thus proving the proposition.
\end{proof}

\section{Syntax Tree Transformations}
\label{sec:sufCond}

In this section, we list some results on syntax trees and assignments that can be applied to find a distribution of formulas and team partition whose parallel satisfaction implies satisfaction of the original formula.

\subsection{Graph operators for decomposition}
We begin by defining decomposition operators for CaTL syntax trees. We focus on three types of operations: \emph{pruning}, or removing a section of the syntax tree; \emph{substitution} of a portion of the tree with a new set of nodes; and \emph{parallelization}, or splitting the tree into multiple trees that can be executed in parallel. Satisfaction of the CaTL formula of the modified tree will imply satisfaction of the original CaTL formula.

\begin{dfn}[Transformation]
Let $\tau$ be a subtree of $\mathcal{T}$.
A transformation rule $\tau \leadsto \{\tau'_1,\ldots, \tau'_N\}$ produces $N$ copies of the tree $\mathcal{T}$ with the subtree $\tau$ replaced by $\tau'_k$, respectively.
Transformations can not introduce new tasks, i.e., $\Lambda'_k \subseteq \Lambda$,
where $\Lambda$ and $\Lambda'_k$ are the leaves of $\mathcal{T}$ and the transformed trees $\mathcal{T}'_k$, respectively. 
An assignment $\alpha$ on $\mathcal{T}$ induces assignments $\alpha'_k$ on the transformed trees $\mathcal{T}'_k$ such that $\alpha'_k(\lambda) = \alpha(\lambda)$ for all $\lambda \in \Lambda'_k$. 
\end{dfn}

In the following, we denote by $\langle v \mid \tau_1, \ldots, \tau_N \rangle$ a syntax tree with root $v$ and subtrees $\tau_1, \dots, \tau_N$.
Given a node $v$ of a tree $\mathcal{T}$, we denote the subtree rooted at $v$ as $\mathcal{T}[v]$.

Let $\mathcal{T}$ be a syntax tree, $v$ a node of $\mathcal{T}$, and $\tau$ a subtree. We define three transformations:
\begin{enumerate}
    \item pruning of siblings: $v \xrightarrow{r} v'$ creates a tree with all siblings of $v' \in Ch(v)$ removed, i.e., $\langle v \mid \mathcal{T}[c], c \in Ch(v) \rangle \leadsto \{ \langle v \mid \tau_{v'} \rangle \}$;
    \item substitution: $\tau \xrightarrow{s} \tau'$ creates a tree with $\tau$ replaced by $\tau'$ on the same tasks, i.e., $\tau \leadsto \{\tau'\}$;
    \item parallelization: $\langle v \mid \tau_1, \dots, \tau_N \rangle \leadsto \{\tau_1,\ldots,\tau_N\}$ creates $N$ trees where the subtree rooted at $v$ is replaced by its subtrees $\tau_1, \dots, \tau_N$.
\end{enumerate}

We now proceed to determining a set of sufficient conditions for pruning, substitution and parallelization.

\subsection{Conditions for Pruning}
\paragraph{Disjunction $\vee$} Note that the capability excess at a node in the syntax tree corresponding to disjunction is non-negative if and only if at least one if its children has a non-negative capability excess.
In other words, satisfaction of the exclusive disjunction $\dot{\vee}$ is sufficient for the satisfaction of disjunction $\vee$.   This is encapsulated in the following proposition.

\begin{prop}
\label{disjProp}
Let $v=\lor$ be a disjunction node in a syntax tree $\mathcal{T}$,  $v' \in Ch(v)$, and $\alpha$ an assignment such that $ce(\alpha, v') \geq 0$.
For any assignment $\alpha$, $\alpha' \models_e \mathcal{T}' \Rightarrow \alpha \models_e \mathcal{T}$, where $\mathcal{T}'$ is obtained by applying the pruning rule $v \xrightarrow{r} v'$, and $\alpha'$ is the induced assignment from $\alpha$.
\end{prop}
\begin{proof}
If $\alpha' \models_e \mathcal{T}',$ then at least one child of $v$ has non-negative capability excess and, thus, $v$ has non-negative capability excess.  Further, since the rest of $\mathcal{T}$ is identical to $\mathcal{T}'$, non-negative excess of $v$ and eligibility of $\alpha'$ implies eligibility of $\alpha$.
\end{proof}

\subsection{Conditions for Substitution}

\paragraph{Until $\until{a,b}$}  The formula $\phi_1 \until{a,b} \phi_2$ is in general difficult to parallelize.  However, we can substitute the until operator with the more conservative formula $\alw{0,b} \phi_1 \land  \even{a,b} \phi_2$ that is amenable to parallelization.

\begin{prop}
\label{untProp}
Let $v=\until{a,b}$ be an until node in a syntax tree $\mathcal{T}$.
If we apply the substitution
\begin{equation}\label{eq:subunt}
\langle \until{a,b} \mid \phi_1, \phi_2 \rangle \xrightarrow{s} \langle \land \mid \langle \alw{0,b} \mid \phi_1 \rangle,  \langle \even{a,b} \mid \phi_2\rangle \rangle,
\end{equation}
then $\alpha' \models_e \mathcal{T}' \Rightarrow \alpha \models_e \mathcal{T}$,
where $\mathcal{T}'$ is obtained from $\mathcal{T}$ by applying the substitution rule, and $\alpha'$ is the induced assignment from $\alpha$.
\end{prop}

\begin{proof}
If $\alpha'\models_e \mathcal{T}'$, then both $ce(\alpha,\phi_1)\geq 0$ and $ce(\alpha,\phi_2)\geq 0$, according to the recursive relationship in~\eqref{eq:ceRec}.
Therefore, since the capability excess of $\until{a,b}$ is the minimum of the capability excess of its children ($\phi_1$ and $\phi_2$), the capability excess of the original formula is non-negative.
This implies $\alpha\models_e\mathcal{T}$.
\end{proof}

\paragraph{Conjunction $\wedge$ with upstream temporal operators}  Temporal operators followed by a conjunction may be parallelized if the assignments of children in the conjunction are disjoint.

\begin{prop}
\label{tempProp}
Let $v=\wedge$ be a node in a syntax tree $\mathcal{T}$, and let 
$v_1\ldots v_n$ be the path such that $v_n = Par(v)$, $v_{k-1} = Par(v_{k}), \forall k =2 \ldots n$, and $v_k \in \{\alw{a,b},\even{a,b}\}$.
With slight abuse of notation, if we make the substitution
\begin{equation}\label{eq:suband}
    \begin{aligned}
        \langle v_1\ldots v_n v \mid {}& \mathcal{T}[c], c \in Ch(v) \rangle\\
        &\xrightarrow{s} \langle v \mid \langle v'_1 \ldots v'_n \mid \mathcal{T}[c]\rangle, c \in Ch(v) \rangle
    \end{aligned}
\end{equation}
where $v'_k$ corresponds to $\square$ with the same time bounds as $v_k$, then 
$\alpha'\models_e\mathcal{T}' \Rightarrow \alpha\models_e \mathcal{T}$. Again, $\mathcal{T}'$ is obtained from $\mathcal{T}$ by applying the substitution rule, and $\alpha'$ is the induced assignment from $\alpha$.
\end{prop}
\begin{proof}
If $\alpha'\models_e\mathcal{T}'$, then $ce(\alpha',v) \geq 0$. Further, $ce(\alpha',v) = \min_{c\in Ch(v)} ce(\alpha',c)$.
Therefore, $\min_{c\in Ch(v)} ce(\alpha',c)\geq 0$.
Since temporal operators $v_1,\ldots,v_N$ do not modify capability excess according to~\eqref{eq:ceRec}, we have $ce(\alpha,v)=\min_{c\in Ch(v)} ce(\alpha,c)\geq 0$.
This implies $\alpha\models_e\mathcal{T}$.
\end{proof}

\subsection{Conditions for Parallelization}

\paragraph{Conjunction  $\wedge$ at the root} Conjunctions at the root of a syntax tree can be parallelized if the assignments of their children do not overlap.

\begin{prop}
\label{conjProp}
Let $v_0$ be the root in a syntax tree such that $v_0=\wedge$.
Let $\alpha(v')\cap \alpha(v'') =\emptyset$ for all $v',v''\in Ch(v_0)$.
If we parallelize
\begin{equation}\label{eq:parallelize}
    \langle v_0 \mid \tau_1,\ldots,\tau_N \rangle\leadsto\lbrace \tau_1,\ldots,\tau_N\rbrace
\end{equation}
then $(\bigwedge_{i=1}^N \alpha_i \models_e \tau_i) \Rightarrow \alpha \models_e \mathcal{T}$, where $\alpha_i$ is the induced assignment on subtree $\tau_i$ from $\alpha$.
\end{prop}
\begin{proof}
For each subtree $\tau_i$, $\alpha_i\models_e\tau_i$ implies that capability excess is non-negative (i.e., $ce(\alpha_i,\tau_i)\geq 0$). Since the capability excess of conjunction is simply the minimum of its children, then $ce(\alpha,v_0) = \min_{i}ce(\alpha_i,\tau_i)$. For all $i$, this capability excess is non-negative. Therefore, $ce(\alpha,v_0)\geq 0$, implying $\alpha\models_e\mathcal{T}$.
\end{proof}

\begin{example}
Fig.~\ref{fig:decomp} demonstrates the application of Propositions \ref{disjProp}-\ref{conjProp}.
\end{example}

\begin{figure}
\begin{subfigure}[b]{0.75\columnwidth}
\centering 
\resizebox{\linewidth}{!}{
\begin{tikzpicture}[scale=.6,level 1/.style = {sibling distance = 20em, 
    draw, align=center},
    level 2/.style={sibling distance=10em, 
    draw, align=center, level distance=2.2cm},
     level 3/.style = {sibling distance = 10em,  
    draw, align=center, level distance=2.2cm}]
    \node [shape=circle, fill=cyan,label={\small$\{A_1,A_2,A_3,A_6,\ldots, A_{10}\}$}]{$\wedge$}
            child { node [shape=circle, fill=green,label={\small$\{A_1,A_2,A_3\}$}] {$T_1$} }
        child { node [shape=circle, fill=cyan,label={\small$\{A_6,\ldots,A_{10}\}$}] {$\until{a,b}$}
            child { node [shape=circle, fill=green,label={$\small\{A_6\}$}] {$T_3$ } } 
            child { node [shape=circle, fill=cyan,label={\small$\{A_7,\ldots,A_{10}\}$}]  {$\wedge$} 
                child { node [shape=circle, fill=green,label={\small$\{A_7,A_8\}$}] {$T_4$}}
                child { node [shape=circle,fill=green,label={\small$\{A_9$, $A_{10}\}$}] {$T_5$}}}};
              
\end{tikzpicture}}
\caption{\label{disjApp} Application of Proposition~\ref{disjProp}.}
\end{subfigure}

\begin{subfigure}[b]{0.75\columnwidth}
\centering 
\resizebox{\linewidth}{!}{
\begin{tikzpicture}[scale=.6,level 1/.style = {sibling distance = 15em, 
    draw, align=center, level distance=2.5cm},
    level 2/.style={sibling distance=10em, 
    draw, align=center, level distance=2.5cm},
     level 3/.style = {sibling distance = 10em,  
    draw, align=center, level distance=2.2cm}]
    \node [shape=circle, fill=cyan,label={\small$\{A_1,A_2,A_3,A_6,\ldots, A_{10}\}$}]{$\wedge$}
        child { node [shape=circle, fill=green,label={\small$\{A_1,A_2,A_3\}$}] {$T_1$} }
        child { node [shape=circle, fill=cyan,label={\small$\{A_6\}$}] {$\alw{0,b}$}
            child { node [shape=circle, fill=green,label={$\small\{A_6\}$}] {$T_3$ } }}
        child { node [shape=circle, fill=cyan,label={\small$\{A_7,\ldots,A_{10}\}$}] {$\even{a,b}$}
            child { node [shape=circle, fill=cyan,label={\small$\{A_7,\ldots,A_{10}\}$}]  {$\wedge$} 
                child { node [shape=circle, fill=green,label={\small$\{A_7,A_8\}$}] {$T_4$}}
                child { node [shape=circle,fill=green,label={\small$\{A_9$, $A_{10}\}$}] {$T_5$}}}};
              
\end{tikzpicture}}
\caption{\label{untApp} Application of Proposition~\ref{untProp}.}
\end{subfigure}

\begin{subfigure}[b]{0.75\columnwidth}
\centering 
\resizebox{\linewidth}{!}{
\begin{tikzpicture}[scale=.6,level 1/.style = {sibling distance = 10em, 
    draw, align=center, level distance=2.5cm},
    level 2/.style={sibling distance=10em, 
    draw, align=center, level distance=2.5cm}]
    \node [shape=circle, fill=cyan,label={\small$\{A_1,A_2,A_3,A_6,\ldots, A_{10}\}$}]{$\wedge$}
        child { node [shape=circle, fill=green,label={\small$\{A_1,A_2,A_3\}$}] {$T_1$} }
        child { node [shape=circle, fill=cyan,label={\small$\{A_6\}$}] {$\alw{0,b}$}
            child { node [shape=circle, fill=green,label={$\small\{A_6\}$}] {$T_3$ } }}
        child { node [shape=circle, fill=cyan,label={\small$\{A_7$, $A_{10}\}$}] {$\alw{a,b}$}
                child { node [shape=circle, fill=green,label={\small$\{A_7,A_8\}$}] {$T_4$}}}
        child { node [shape=circle, fill=cyan,label={\small$\{A_9$, $A_{10}\}$}] {$\alw{a,b}$}
                child { node [shape=circle,fill=green,label={\small$\{A_9$, $A_{10}\}$}] {$T_5$}}};
              
\end{tikzpicture}}
\caption{\label{tempApp} Application of Proposition~\ref{tempProp}.}
\end{subfigure}

\begin{subfigure}[b]{0.75\columnwidth}
\centering 
\resizebox{\linewidth}{!}{
\begin{tikzpicture}[scale=1.2,level 1/.style = {sibling distance = 15em, 
    draw, align=center, level distance=2.4cm},
    level 2/.style={sibling distance=10em, 
    draw, align=center, level distance=2.4cm},
     level 3/.style = {sibling distance = 10em,  
    draw, align=center, level distance=2.4cm}]
        \node [shape=circle, fill=green,label={\small$\{A_1,A_2,A_3\}$}] {$T_1$};
            \end{tikzpicture}
            \begin{tikzpicture}
             \node [shape=circle, fill=cyan,label={$\small\{A_6\}$}] {$\alw{0,b}$}
            child { node [shape=circle, fill=green,label={$\small\{A_6\}$}] {$T_3$ } } ;
            \end{tikzpicture}
            \begin{tikzpicture}
            \node [shape=circle, fill=cyan,label={\small$\{A_7,A_8\}$}] {$\alw{a,b}$}
                child { node [shape=circle, fill=green,label={\small$\{A_7,A_8\}$}] {$T_4$}};
                            \end{tikzpicture}
            \begin{tikzpicture}
            \node [shape=circle, fill=cyan,label={\small$\{A_9,A_{10}\}$}] {$\alw{a,b}$}
                child { node [shape=circle,fill=green,label={\small$\{A_9$, $A_{10}\}$}] {$T_5$}};
\end{tikzpicture}}
\caption{ \label{conjApp} Application of Proposition \ref{conjProp}.}
\end{subfigure}
\caption{\label{fig:decomp} Decomposition of abstract syntax tree and assignments by applying Propositions \ref{disjProp}-\ref{conjProp}.}
\end{figure}
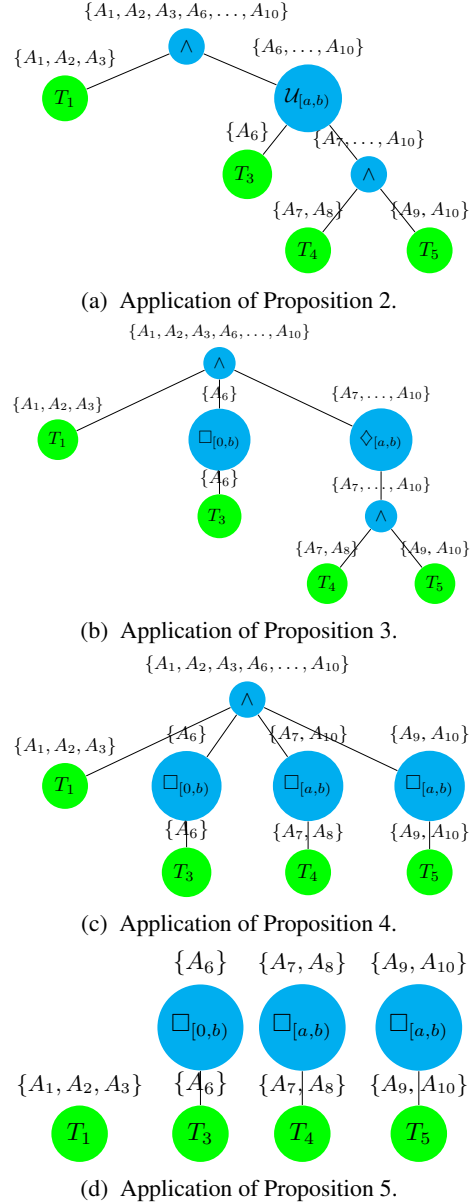

\section{Agent Assignments as Satisfiability Modulo Theory (SMT)}
\label{sec:sat}
Now, we consider the problem of finding an assignment $\alpha$ for a given CaTL formula $\phi$ such that applying Propositions \ref{disjProp}-\ref{tempProp} will result in a set of subformulas and subteams that can be decomposed. 
We consider the problems of determining eligibility, pruning, substituting, and parallelizing. 
Eligibility variables capture whether an assignment is eligible according to Definition~\ref{def:elg}. 
Independence variables keep track of any overlap between assignments of agents to tasks, thereby helping determine the feasibility of substitution and parallelization. 
In the definitions below, these Boolean variables are indicated by the terms $Elg$ and $Ind$, respectively.

\subsection{CaTL Formula Encoding}\label{sec:smt_encode}

\paragraph{Task Assignment Encoding} We first examine 
how an assignment $\alpha$ determines elgibility and independence of tasks at the leaves of the tree.
Consider a set of agents $\{A_j\}_{j \in J}$ with corresponding classes of capabilities $g_j$  and tasks $T_m = (d_m,\pi_m,cp_m)$. 

\begin{itemize}
    \item For each task, the assignment is eligible if the capability excess is non-negative, i.e.,
        $Elg(\alpha,v) := ce(\alpha,v)\geq 0$.
   \item Because tasks are leaves of the syntax tree, they do not inherently require coordination with other tasks. Thus we set independence to be true, i.e.,
        $Ind(\alpha,v) := \mathtt{True}$.
   \end{itemize}

Now, we ascend from the leaves of the tree to the root.  We describe what to do once we encounter each operator.

\paragraph{Disjunction $\vee$}
As previously noted, an assignment for disjunction is eligible if at least one of its children has an eligible assignment. We let the solver select which child is chosen and track the eligibility and independence for that child.
\begin{itemize}
    \item The eligibility of a disjunction node is inherited from the chosen child node, i.e., 
        $Elg(\alpha,v) := Elg(\alpha,Choose(\alpha,v))$,
    where $Choose(\alpha,v)$ is a function that selects a child of $v$ with positive eligibility. It plays the role of $\beta$ as defined in Sec.~\ref{sec:ast}.
    \item Independence is inherited from the chosen child node, i.e.,
        $Ind(\alpha,v):= Ind(\alpha,Choose(\alpha,v))$.
\end{itemize}

\paragraph{Temporal operators $\even{t_1,t_2}$ or $\alw{t_1,t_2}$}
Because these temporal operators have only a single child, their encoding is straightforward.

\begin{itemize}
    \item Eligibility is inherited from the child node, i.e.,
        $Elg(\alpha,v) := Elg(\alpha,Ch(v))$.
    \item Independence of tasks is inherited from the child node, i.e.,
        $Ind(\alpha,v):= Ind(\alpha,Ch(v))$.
\end{itemize}

\paragraph{Until $\until{t_1,t_2}$ or Conjunction $\wedge$}
For until and conjunction, we must track eligibility of the children of each node and whether its children are independent.
\begin{itemize}
    \item Eligibility is determined by the conjunction of the children of $v$, i.e.,
        $Elg(\alpha,v) := \bigwedge_{v'\in Ch(v)} Elg(\alpha,v')$.
\item Independence is determined based on the intersection of assignments of child nodes, i.e.,
        $Ind(\alpha,v):= \bigcap_{v'\in Ch(v)}\alpha(v')=\emptyset$.
\item For conjunction at the root, we allow non-binary conjunction. Therefore, we need to consider a slightly different encoding of $Ind$. With slight abuse of notation, we check for pairwise independence between subtrees at the root, i.e.,
        $Ind(\alpha,v_0,v',v'') := (\alpha(v')\cap\alpha(v'')=\emptyset)$
    pairwise for all $v', v''\in Ch(v_0)$, $v' \neq v''$.    
\end{itemize}

\paragraph{Top level encoding}
We encode the total formula as 
\begin{equation}\label{eq:top}
    \begin{aligned}
        \sum_{v \in \mathcal{T}} Cost(Ind(\alpha, v)))\\
        \text{s.t. } Elg(\alpha,v_0)
    \end{aligned}
\end{equation}
where $Cost(\cdot)$ is a cost function that is negative for any values of $Ind$ that are true, and increases in magnitude towards the top of the tree. 
Checking for root eligibility ensures an eligible assignment in all downstream nodes. Finally, minimizing the independence costs searches for a maximally parallelizable assignment.

\subsection{Decomposing Specification using Assignment and SMT}
Given an assignment from the preceding SMT problem, we wish to decompose the tree as much as possible, returning a set of subformulas and subteams. Here we describe how to obtain these formulas and teams from the SMT solution.

The process is outline in Algorithm~\ref{alg:decomp}. Briefly, the syntax tree and assignment from the SMT are provided as input. The tree is pruned at nodes labeled $\vee$ for any children not selected in the SMT (lines 1-2). Then, subtrees are substituted at nodes labeled $\until{a,b}$ or $\wedge$ if the assignments to their children are disjoint (lines 3-5). 
Finally, the tree is parallelized according to the independence of children at the root node (lines 6-7).
The resulting formulas and their corresponding team assignments are extracted and returned. These output formulas and teams can then each by solved by a MILP as in~\citet{Jones2019ScRATCHS}.

\begin{remark}
We note that the tree does not necessarily need to be pruned at $\vee$ nodes. The assignment to downstream nodes from $\vee$ will be empty. It may be the case that agents assigned to the eligible branch can accomplish tasks in the ineligible branch. Thus, the MILP has a greater chance at finding solutions at the cost of computation over a larger formula.
\end{remark}

\begin{algorithm}
\KwIn{Assignment $\alpha$ from the SMT Problem,\\Syntax Tree $\mathcal{T}_{\phi}$}
\KwOut{Set of formulas $\{\phi_i\}_{i \in 1,\ldots,N}$,\\  Corresponding team partition $\{J_i\}_{i \in 1,\ldots,N}$}
\For{$v\in V \vert v=\vee$}{
    prune subtree rooted at $v$\;
    }
\For{$v\in V \vert v\in\{\until{a,b},\wedge\}${\bf and } $v\neq v_0$}{
\If{ $Ind(\alpha,v)=\mathtt{True}$}{
    substitute according to~\eqref{eq:subunt} or~\eqref{eq:suband}\;}
    }
\If{$\exists v',v''\in Ch(v_0) \vert Ind(\alpha,v_0,v',v'')=\mathtt{True}$}{
    $\mathcal{T}\leadsto\lbrace \tau_1,\ldots,\tau_N\rbrace$ according to~\eqref{eq:parallelize}\;}
\For{$\tau_i \in \lbrace \tau_1,\ldots,\tau_N\rbrace$}{
    extract formula $\phi_i$ from subtree $\tau_i$\;
    extract subteam $J_i$ from $\alpha_i$\;
}
    
\Return{$\{\phi_i\}_{i \in 1,\ldots,N}$, $\{J_i\}_{i \in 1,\ldots,N}$}
\caption{Decomposition using assignment $\alpha$ from SMT}\label{alg:decomp}
\end{algorithm}

\remark{If the assignment is eligible and feasible but no satisfying parallel execution exists, we must furthermore add that information to the SMT problem so that the solver does not continue to investigate similar solutions that are unlikely to work. This can be accomplished using the irreducible inconsistent set (IIS). The IIS is computed by most modern solvers, and provides constraints that can be used in the SMT. There are several technical issues that need to be addressed in this process, and we leave it as future work.}\label{IIS}

\subsection{Simulation and Results}

To validate our proposed methodology and evaluate its computational performance, we performed computational experiments.
The SMT problem was coded in Z3~\cite{demoura2008}.
Synthesis was performed using the Gurobi solver~\cite{gurobi}.
Experiments were run in Python 2.7 on Ubuntu 16.04 with a 2.5 GHz Intel i7 processor and 16 GB of RAM.

To evaluate our methodology, we tested the system in an environment with 25 states for varying numbers of agents. 
The specification in~\eqref{examSpec} was used for all experiments, and the number of agents required per task are given in Fig.~\ref{expTab}.
Agents were randomly assigned capabilities from the set $\lbrace (c1),(c2),(c1,c2)\rbrace$.
We evaluated the performance for 10, 20, 30, 40, and 50 agents.
The environment was a fully-connected $5\times 5$ grid, with edge weight of $1$ and randomly assigned region labels.
In general, increasing the number of agents decreases the time to find a solution.
Therefore as the number of agents increased, the number of regions to be serviced was also increased, to maintain the same approximate difficulty of finding a solution.
We calculated the time to find the first feasible solution, as well as the time to find the maximally robust solution.
One hundred simulations were run for each condition, and we set timeout to be $120s$.

Run time results are shown in Fig.~\ref{fig:runtime} and Tables~\ref{tab:feasible} and~\ref{tab:robust}. For the time to first feasible solution, and the time to robust solution, the decomposed system executes faster than the centralized system. The run times displayed include both decomposition time and MILP solution time. 
However, we note that the assignment via SMT does not take robustness into account. Therefore, the robust centralized solution consistently returns a more robust solution than the decomposed solver, which often returns a robustness of $0$. For the decomposed solution, robustness was computed as the minimum of all subformula robustness, meaning the worst-case robustness could be $0$, even if some subteams performed better.

\begin{figure}
    \centering
    \includegraphics[width=0.8\columnwidth]{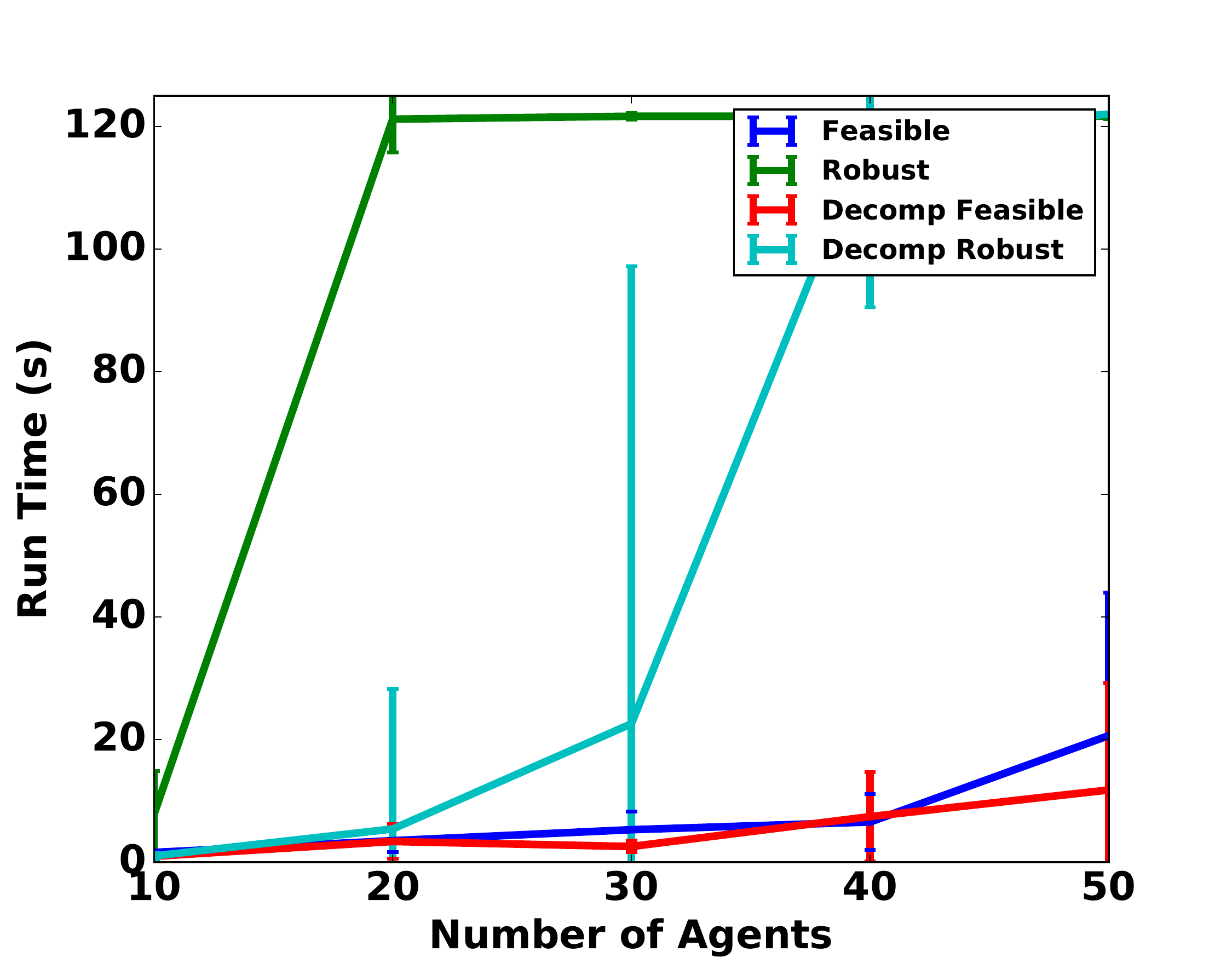}
    \caption{Run time with and without decomposition for varying problem sizes. Simulations performed for first feasible solution and maximally robust solution. 100 random trials were run for each case. Timeout was set at 120s.}\label{fig:runtime}
\end{figure}

\begin{table}
\centering
\caption{Run time for obtaining the first feasible solution. Results are presented as mean / max.}\label{tab:feasible}
\begin{tabular}{ |c|c|c| } 
\hline
 & No Decomposition & Decomposition\\
 \cline{2-3}
 Agents & Time (s) & Time (s) \\
 \hline
 10 & 1.60 / 1.97 &  0.96 / 1.99 \\ 
 20 & 3.51 / 6.50 &  3.38 / 5.94 \\ 
 30 & 5.30 / 11.46 & 2.56 / 4.42 \\ 
 40 & 6.56 / 12.14 & 7.43 / 19.72 \\ 
 50 & 20.65 / 51.25 & 11.78 / 28.13 \\ 
 \hline
\end{tabular}
\end{table}

\begin{table}
\centering
\caption{Run time and robustness for obtaining the optimally robust solution. Results are presented as mean / max. Timeout was set to 120 s.}\label{tab:robust}
\begin{tabular}{ |c|c|c|c|c| } 
\hline
 & \multicolumn{2}{|c|}{No Decomposition} & \multicolumn{2}{|c|}{Decomposition}\\
 \cline{2-5}
 Agents & Time (s) & $\rho$ & Time (s) & $\rho$\\
 \hline
 10 & 7.81 / 25.32 & 0.88 / 2 & 1.02 / 2.54 & 0 / 0 \\ 
 20 & 120 / 120 & 0.79 / 1 & 5.42 / 107.49 & 0 / 0 \\ 
 30 & 120 / 120 & 0.87 / 1 & 22.58 / 120 & 0.01 / 1 \\ 
 40 & 120 / 120 & 0.89 / 1 & 120 / 120 & 0 / 0 \\ 
 50 & 120 / 120 & 0.88 / 1 & 120 / 120 & 0.33 / 1 \\ 
 \hline
\end{tabular}
\end{table}

In Fig.~\ref{fig:decomptime}, we show the proportion of the run time that is used for decomposition versus solving the MILP. For each data point, the run time is dominated by the MILP solver and not the decomposition. This suggests that the decomposition process is efficient. When coupled with the results in Fig.~\ref{fig:runtime}, it suggests that the cost of decomposition is low, while the benefit of decompostion is high.

\begin{figure}
    \centering
    \includegraphics[width=0.8\columnwidth]{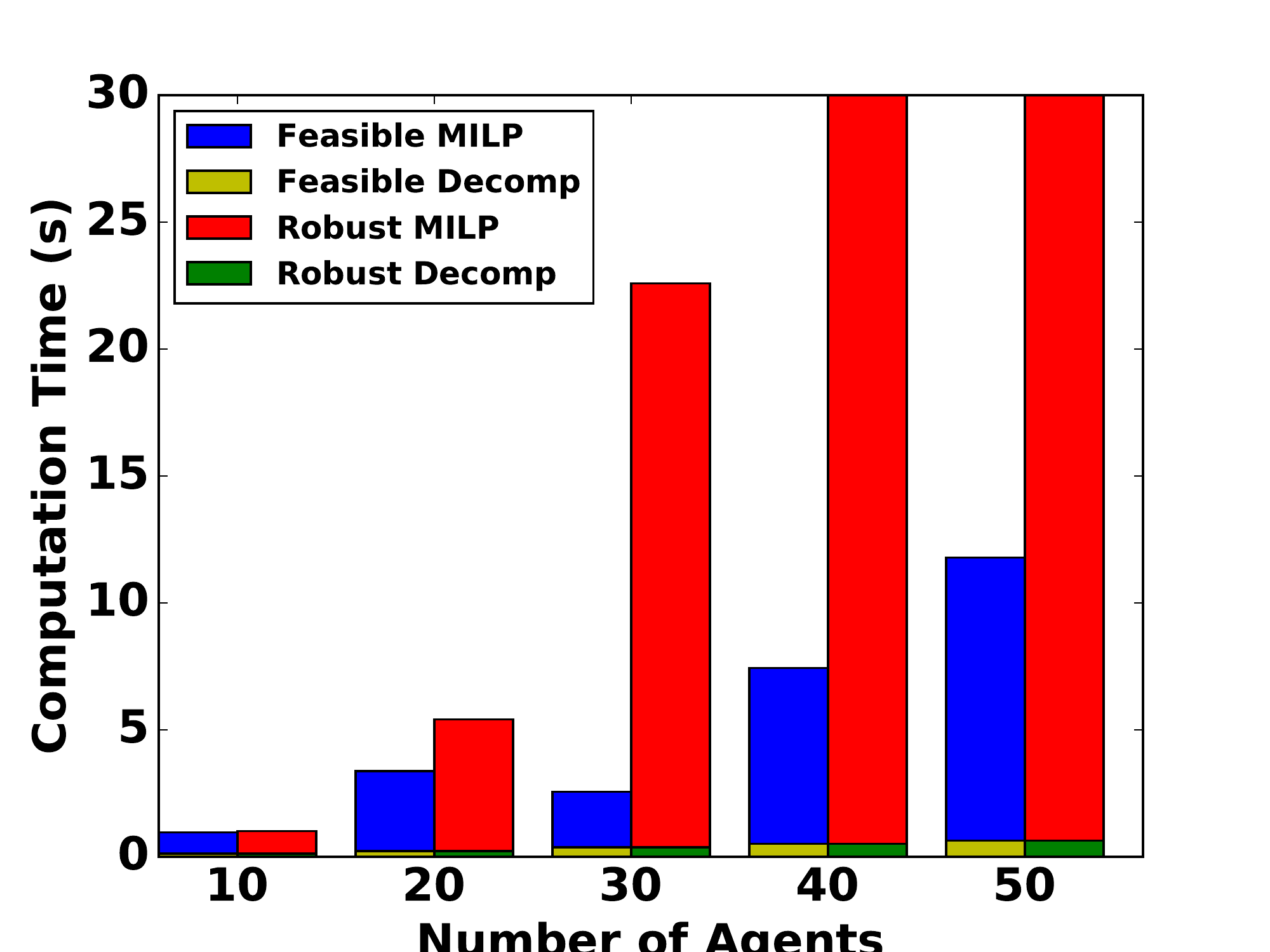}
    \caption{Overall run time broken into decomposition time and MILP time. Results are shown for the feasible and robust MILP solution.}\label{fig:decomptime}
\end{figure}

\subsection{Conclusion}

In this work, we have proposed a method for the automatic decomposition of a team of agents and a formal specification into a set of subteams and sub-specifications. Our method employs SMT to find a feasible decomposition that we then solve in a distributed manner using a set of MILPs. This method significantly reduces the run time over a centralized approach. It represents a promising first step towards speeding up planning for large heterogeneous teams.

There are several avenues of future work. 
First, the robustness of the decomposed solutions is significantly lower than for the centralized solution. 
By using capability excess as part of our cost function, we may be able to find a decomposition that is more robust.
It may also be possible that an assignment meets our criteria but has no feasible solution via the MILP (i.e., agents cannot service their required tasks according to their timed deadlines). 
One possible solution to that problem is to use the IIS (see Remark~\ref{IIS}) in feedback with the SMT problem to remove any infeasible conditions. 
We may also be able to incorporate properties of the environment or agent locations into the assignment problem.

\bibliography{RSSBib2,RSSBib,DecompRefs}

\end{document}